\documentclass{article}

\RequirePackage[OT1]{fontenc}
\usepackage[preprint]{neurips_2020}
\usepackage{amsmath}
\usepackage{amsfonts}
\usepackage{amssymb}
\usepackage{graphicx}
\usepackage{tikz}
\usepackage{dsfont}
\usepackage{empheq}
\usepackage[protrusion=true,final]{microtype}
\usepackage{booktabs}
\usepackage{algorithm}
\usepackage{algpseudocode}
\usepackage{float}
\usepackage{blkarray}
\usepackage{mathtools}
\usepackage{pgfplots}
\usepackage{caption}
\usepackage{subcaption}
\usepackage{hyperref}

\DeclareMathOperator{\Var}{Var}

\newcommand{\argmin}{\operatornamewithlimits{\arg\min}}

\DeclarePairedDelimiterX{\infdivx}[2]{(}{)}{%
  #1\;\delimsize\|\;#2%
}
\newcommand{\infdiv}{D_\mathit{KL}\infdivx}

\newcommand{\BlackBox}{\rule{1.5ex}{1.5ex}}  %
\newenvironment{proof}{\par\noindent{\bf Proof\ }}{\hfill\BlackBox}
 
\newtheorem{theorem}{Theorem}
\newtheorem{lemma}[theorem]{Lemma} 
\newtheorem{proposition}[theorem]{Proposition}

\newtheorem{assumption}{Assumption}

\begin{document}

\title{Nearest Neighbour Based Estimates of Gradients:\\ Sharp Nonasymptotic Bounds and Applications}

\author{Guillaume Ausset \\
        LTCI, T\'el\'ecom Paris, Institut Polytechnique de Paris\\
        BNP Paribas\\
        \texttt{guillaume.ausset@telecom-paris.fr} \\
        \And
        Stephan Cl\'emen\c{c}on \\
        LTCI, T\'el\'ecom Paris, Institut Polytechnique de Paris\\     
        \texttt{stephan.clemencon@telecom-paris.fr} \\
        \And
        Fran\c{c}ois Portier \\
        LTCI, T\'el\'ecom Paris, Institut Polytechnique de Paris \\
        \texttt{francois.portier@telecom-paris.fr}}

\maketitle

\begin{abstract} Motivated by a wide variety of applications, ranging from stochastic optimization to dimension reduction through variable selection, the problem of estimating gradients accurately is of crucial importance in statistics and learning theory. We consider here the classic regression setup, where a real valued square integrable r.v. $Y$ is to be predicted upon observing a (possibly high dimensional) random vector $X$ by means of a predictive function $f(X)$ as accurately as possible in the mean-squared sense and study a nearest-neighbour-based pointwise estimate of the gradient of the optimal predictive function, the regression function $m(x)=\mathbb{E}[Y\mid X=x]$. Under classic smoothness conditions combined with the assumption that the tails of $Y-m(X)$ are sub-Gaussian, we prove nonasymptotic bounds improving upon those obtained for alternative estimation methods. Beyond the novel theoretical results established, several illustrative numerical experiments have been carried out. The latter provide strong empirical evidence that the estimation method proposed works very well for various statistical problems involving gradient estimation, namely dimensionality reduction, stochastic gradient descent optimization and quantifying disentanglement. 
\end{abstract}

\section{Introduction}

In this paper, we place ourselves in the usual regression setup, one of the flagship predictive problems in statistical learning. Here and throughout, $(X,Y)$ is a pair of random variables defined on the same probability space $(\Omega,\; \mathcal{F},\; \mathbb{P})$ with unknown probability distribution $P$: the r.v. $Y$ is real valued and square integrable, whereas the (supposedly continuous) random vector $X$ takes its values in $\mathbb{R}^D$, with $D\geq 1$, and models some information \textit{a priori} useful to predict $Y$. Based on a sample $\mathcal{D}_n=\{ (X_1,Y_1),\; \ldots,\; (X_n,Y_n) \}$ of $n\geq 1$ independent copies of the generic pair $(X,Y)$, the goal pursued is to build a Borelian mapping $f:\mathbb{R}^d\rightarrow \mathbb{R}$ that produces, in average, a good prediction $f(X)$ of $Y$. Measuring classically its accuracy by the squared error, the learning task then boils down to finding a predictive function $f$ that is solution of the risk minimization problem $\min_{f}\mathcal{R}_P(f)$, where 
\begin{equation}\label{eq:true_risk}
\mathcal{R}_P(f)=\mathbb{E}\left[\left(Y-f(X)\right)^2\right].
\end{equation}
Of course, the minimum is attained by the regression function $m(X)=\mathbb{E}[Y\mid X]$, which is unknown, just like $Y$'s conditional distribution given $X$ and the risk \eqref{eq:true_risk}. The empirical risk minimization (ERM) strategy consists in solving the optimization problem above, except that the unknown distribution $P$ is replaced by an empirical estimate based on the training data $\mathcal{}D_n$, such as the raw empirical distribution $\hat{P}_n=(1/n)\sum_{i\leq n}\delta_{X_i}$ typically, denoting by $\delta_x$ the Dirac mass at any point $x$, and minimization is restricted to a class $\mathcal{F}$ supposed to be rich enough to include a reasonable approximant of $m$ but not too complex (\textit{e.g.} of finite {\sc VC} dimension) in order to control the fluctuations of the deviations between the empirical and true distributions uniformly over it. Under the assumption that the random variables $Y$ and $f(X)$, $f\in \mathcal{F}$, have sub-Gaussian tails, the analysis of the performance of empirical risk minimizers (\textit{i.e.} predictive functions obtained by least-squares regression) has been the subject of much interest in the literature, see \textit{e.g.} \cite{gyorfiDistributionFreeTheoryNonparametric2002}, \cite{massartConcentrationInequalitiesModel2007}, \cite{boucheronConcentrationInequalitiesNonasymptotic2013} or \cite{lecueLearningSubgaussianClasses2016}  (and refer to \textit{e.g.} \cite{lugosiRiskMinimizationMedianofmeans2016} for alternatives to the ERM approach in non sub-Gaussian situations).

In this paper, we are interested in estimating accurately the (supposedly well-defined) gradient $\nabla m(x)$ by means of the popular $k$ nearest neighbour ($k$-NN) approach, see \textit{e.g.} Chapter in \cite{devroyeProbabilisticTheoryPattern1996a} or \cite{biauLecturesNearestNeighbor2015}. The \textit{gradient learning} issue has received increasing attention in the context of local learning problems such as classification or regression these last few years, see \textit{e.g.} \cite{mukherjeeEstimationGradientsCoordinate2006,mukherjeeLearningCoordinateCovariances2006}. Because it provides a valuable information about the local structure of a dataset in a high-dimensional space, an accurate estimator of the gradient of a predictive function can be used for various purposes such as dimensionality reduction or variable selection (see \textit{e.g.} \cite{hristacheStructureAdaptiveApproach2001, hristacheDirectEstimationIndex1998,xiaAdaptiveEstimationDimension2002,xiaConstructiveApproachEstimation2007,dalalyanNewAlgorithmEstimating2008,yeLearningSparseGradients2012}), the partial derivative w.r.t. a given variable being a natural indicator of its importance regarding prediction. The previous references are all concerned with outer-products of gradients so as to recover some dimension-reduction subspace.
Estimators of the gradients have also been proposed for zeroth-order optimization (see \textit{e.g.} \cite{wangStochasticZerothorderOptimization2018}) and can benefit from good convergence properties.

Whereas the use of standard nonparametric methods for gradient estimation is documented in the literature (see \citep{fanLocalPolynomialModelling1996,delecroixNonparametricEstimationRegression1996,debrabanterDerivativeEstimationLocal2013} for the use of local-polynomial with kernel smoothing techniques, \citep{gasserEstimatingRegressionFunctions1984} for the so-called Gasser-Muller alternative and \citep{zhouDerivativeEstimationSpline2000} for the use of regression spline), it is the purpose of the present article to investigate the performance of an alternative local averaging method, the popular $k$-NN method. As it provides piecewise constant estimates, it is easier to conceptualize for the practitioner and, more importantly; the neighbourhoods determined by the parameter $k$ are data-driven and often more consistent than those defined by the bandwidth in the kernel setting, especially in high dimensions.

Here we investigate the behaviour of the estimator of the (supposedly sparse) gradient of the regression function at a given point $x\in \mathbb{R}^D$, obtained by solving a regularized local linear version of the $k$-NN problem with a Lasso penalty. Precisely, nonasymptotic bounds for the related estimation error are established. Whereas $k$-NN estimators of the regression function have been extensively analysed from a nonasymptotic perspective (see \textit{e.g.} \cite{jiangNonAsymptoticUniformRates2019} and the references therein), the result stated in this paper is the first of this type to the best of our knowledge. 

The relevance of the approach promoted is then illustrated by several applications. A variable selection algorithm that exploits the local nature of the gradient estimator proposed is first exploited to refine the popular random forest algorithm (see \cite{breimanRandomForests2001a}): by exploiting the node estimate of the gradient we are able to better direct the choice of cuts. Very simple to implement and accurate, as supported by the various numerical experiments carried out, it offers an attractive and flexible alternative to existing traditional methods such as PCA or the more closely related method of \cite{dalalyanNewAlgorithmEstimating2008}, allowing for a local reduction of the dimension rather than implementing a global preprocessing of the data.
We next show how a rough statistical estimate of the gradient of any smooth objective function based on the estimation principle previously analysed in the context of regression can be exploited in a basic gradient descent algorithm, we exploit the local structure of the algorithm to be able to reuse past computations in order to calculate our estimator and jump to a better local minimum at each gradient step as well.
Finally, we give an example of the usefulness of a sparse gradient estimate when one believes the gradient to be truly sparse: we use our estimator to retrieve the direction of interest for a specific attribute inside a disentangled representation and show how this can be used as an \textit{ad hoc} measure of disentanglement.

The article is organized as follows. In section \ref{sec:background}, the estimation method and the assumptions involved in the subsequent analysis are listed. The main theoretical results of the paper are stated in section \ref{sec:main}, while several applications of the estimation method promoted are described at length and illustrated by numerical experiments in section \ref{sec:exp}. Some concluding remarks are collected in section \ref{sec:conclusion} and technical proofs, as well as additional numerical results, are postponed to the Supplementary Material.

\section{Background - The Estimation Framework}\label{sec:background}
We place ourselves in the nonparametric regression setup described in the previous section. Here and throughout, the indicator function of any event $\mathcal{E}$ is denoted by $\mathds{1}_{\mathcal{E}}$, the cardinality of any finite set $E$ by $\# E$. By $\vert\vert x\vert\vert=\max\{\vert x_1\vert,\; \ldots\; \vert x_D\vert\}$,  $\vert\vert x\vert\vert_1=\vert x_1\vert+\ldots+\vert x_D\vert$ and $\vert\vert x\vert\vert_2=\sqrt{ x_1^2+\ldots+\ x_D^2}$ are meant the $\ell_{\infty}$-norm, the $\ell_1$-norm and the $\ell_2$-norm of any vector $x=(x_1,\; \ldots,\; x_D)$ in $\mathbb{R}^D$.  Any vector $x$ in $\mathbb{R}^D$ is identified as a column vector, the transpose of any matrix $M$ is denoted by $M^\intercal$ and $\mathcal{B}(x, \tau)=\{z\in \mathbb{R}^D:\; \vert\vert x-z\vert\vert \leq \tau \}$ is the (closed) ball of center $x\in \mathbb{R}^D$ and radius $\tau>0$.

\noindent{\bf $k$-NN estimation methods in regression.} Let $x\in \mathbb{R}^D$ be fixed and $k\in\{1,\; \ldots,\; n\}$. Define 
\begin{align*}
    \hat \tau_{k} (x) {=} \inf \{\tau\geq 0 \,:\, \sum_{i=1} ^ n  \mathds{1}_{\{X_i \in \mathcal{B}(x,\tau)\}}  \geq  k \},
\end{align*}
which quantity is referred to as the $k$-NN radius. Indeed, observe that, equipped with this notation, $\mathcal{B}(x, \hat \tau_{k} (x))$ is the smallest ball with center $x$ containing $k$ points of the sample $\mathcal{D}_n$ and the mapping $\alpha\in (0,1] \mapsto \hat \tau_{\alpha n} (x)  $ is the empirical quantile function related to the sample $\{\|x-X_1\|,\; \ldots,\; \|x-X_n\|\}$. The rationale behind $k$-NN estimation in the regression context is simplistic, the method consisting in approximating $m(x)=\mathbb{E}[Y\mid X=x]$ by $\mathbb{E}[Y\mid X\in \mathcal{B}(x, \tau)]$, the mapping $m$ being assumed to be smooth at $x$, and computing next the empirical version of the approximant (\textit{i.e.} replacing the unknown distribution $P$ by the raw empirical distribution). This yields the estimator
\begin{equation}\label{eq:rawNN}
 \hat m_{k} (x)  = \frac 1 k  \sum_{i : X_i \in \mathcal{B}(x, \hat{\tau}_k(x))}  Y_i     ,
\end{equation} 
usually referred to as the standard $k$-nearest neighbour predictor at $x$. Of course, the mapping $x\in \mathbb{R}^D\mapsto \hat m_{k} (x) $ is locally/piecewise constant, just like $x\in \mathbb{R}^D\mapsto \hat{\tau}_{k} (x) $.
The local average $ \hat m_{k} (x) $ can also be naturally expressed as
\begin{equation}
  \hat m_{k}  (x)=   \argmin_{ m  \in \mathbb{R}} \sum_{i : X_i \in \mathcal{B}(x, \hat{\tau}_k(x))} (Y_i - m )^2 .\label{loco}
\end{equation}
For this reason, the estimator \eqref{eq:rawNN} is sometimes referred to as the \textit{local constant} estimator in the statistical literature. Following in the footsteps of the approach proposed in \cite{fanDesignadaptiveNonparametricRegression1992}, the estimation of the regression function at $x$ can be refined by approximating the supposedly smooth function $m(z)$ around $x$ in a linear fashion, rather than by a local constant $m$, since we have $m(z) = m(x) + \nabla m(x)^\intercal (z-x) + o(\lVert z-x \rVert)$ by virtue of a first order Taylor expansion. For any point $X_i $ close to $x$, one may write $m(X_i) \simeq  m + \beta^\intercal (X_i-x)$ and
the \textit{local linear} estimator of $m(x)$ and the related estimator of the gradient $\beta(x)=\nabla m(x)$ are then defined as
\begin{align}
     \argmin_{ (m, \beta) \in \mathbb{R}^{D + 1}} \sum_{i : X_i \in \mathcal{B}(x, \hat{\tau}_k(x))} (Y_i - m - \beta^\intercal (X_i - x))^2 .\label{ll}
\end{align}
Because of its reduced bias, the local linear estimator (the first argument of the solution of the optimization problem above) can improve upon the local constant estimator \eqref{eq:rawNN} in moderate dimensions. However, when the dimension $D$ increases, its variance becomes large and the design matrix of the regression problem is likely to have small eigenvalues, causing numerical difficulties.  For this reason, we introduce here a lasso-type regularized version of~\eqref{ll}, namely
\begin{align}
    (\tilde m_{k}   (x) , \tilde \beta_k (x)) \in \argmin_{(m, \beta)\in \mathbb{R}^{D+ 1}} \sum_{i : X_i \in \mathcal{B}(x, \hat{\tau}_k(x))} (Y_i - m - \beta^\intercal (X_i - x))^2 + \lambda \lVert \beta \rVert_1 ,\label{lll}
\end{align}
where $\lambda>0$ is a tuning parameter governing the amount of $\ell_1$-complexity penalization.  For the moment, we let it be a free parameter and will propose a specific choice in the next section. Focus is here on the gradient estimator $\tilde \beta_k (x)$, \textit{i.e.} the second argument in \eqref{lll}. In the subsequent analysis, nonasymptotic bounds are established for specific choices of $\lambda$ and $k$. The following technical assumptions are required.

\noindent {\bf Technical assumptions.} The hypothesis formulated below permits us to relate the volumes of the balls $\mathcal{B}(x,\; \tau)$ to their probability masses, for $\tau$ small enough. 

\begin{assumption}\label{cond:density}
There exists $\tau_0>0$ such that restriction of $X$'s distribution on $B(x, \tau_0)$ has a bounded density $f_X$, bounded away from zero, with respect to Lebesgue measure:
\begin{equation*}
b_f = \inf _{y\in B(x, \tau_0)}  f_X (y)>0 \text{ and }U_f = \sup _{y\in B(x, \tau_0)}  f_X (y)< +\infty.
\end{equation*}
Suppose in addition that $U_f/ b_f\leq 2$.
\end{assumption}
The constant $2$ involved in the condition above for notational simplicity can be naturally replaced by any constant $1+\gamma$, with $\gamma>0$. The next assumption, useful to control the variance term, is classic in regression, it stipulates that we have $Y=m(X)+\varepsilon$, with a sub-Gaussian residual $\varepsilon$ independent from $X$.
\begin{assumption}\label{cond:sub_gaussian_inovation}
The zero-mean and square integrable r.v. $\varepsilon = Y-m(X)$ is independent from $X$ and is sub-Gaussian with parameter $\sigma^2>0$, \textit{i.e.} $\forall \lambda\in \mathbb R$, $\mathbb{E} [\exp ( \lambda \varepsilon ) ] \leq \exp( - \sigma^2 \lambda^2/2) $.  
\end{assumption}

In order to control the bias error when estimating the gradient $\beta(z)=\nabla m(z)$ of the regression function at $x$, smoothness conditions are naturally required.

\begin{assumption}\label{cond:lip2}
The function $m(z)$ is differentiable on $\mathcal{B}(x, \tau_0)$ with gradient $\beta(z)=\nabla m(z)$ and there exists $L_2>0$ such that for all $z \in \mathcal{B}(x, \tau_0)$,
\begin{align*}
|m (z) - m (x) - \beta(x)  (z-x) | \leq  L_2\|z-x\|^2 .
\end{align*}
\end{assumption}

Finally, a Lipschitz regularity condition is required for the density $f_X$.

\begin{assumption}\label{cond:lip3}
The function $f_X$ is $L$-Lipschitz at $x$ on $\mathcal{B}(x, \tau_0)$, \textit{i.e.} there exists $L>0$ such that for all $z \in B(x, \tau_0)$,
\begin{align*}
|f_X (z) - f_X(x) | \leq  L\|z-x\| .
\end{align*}
\end{assumption}

We point out that, as the goal of this paper is to give the main ideas underlying the use of the $k$-NN methodology for gradient estimation rather than carrying out a fully general analysis,  the $\ell_\infty$-norm is considered here, making the study of $\ell_1$ regularization easier. The results of this paper can be extended to other norms at the price of additional work.

\section{Main result - Rate Bounds for the $k$-NN based Gradient Estimator}\label{sec:main}

The main theoretical result of the present paper is now stated and further discussed. Under the hypotheses listed in the previous section and for specific choices of $k$ and $\lambda$, it provides a nonasymptotic bound for the estimator $\tilde{\beta}_k(x)$ of the gradient $\beta(x)=\nabla m(x)$ at $x$ given by \eqref{lll}. Whereas nonasymptotic bounds for $k$-NN estimators of the regression function have been established under various smoothness assumptions (see \textit{e.g.} \cite{jiangNonAsymptoticUniformRates2019} or \cite{kpotufeKNNRegressionAdapts2011}), no nonasymptotic study of $k$-NN based estimator of the gradient of the regression function is documented in the literature. To the best of our knowledge, the result proved in this article is the first of this nature. Two key quantities are involved in the upper confidence bound given in Theorem \ref{th:gradient_ell2}, the (deterministic) radius 
\begin{align*}
\overline{\tau} _ k =  \left (\frac{ 2  k }{ n b_f2^D}  \right)^{ 1/ D} ,
\end{align*}
that upper bounds the $k$-NN radius on an event holding true with large probability, as well as the cardinality of the so called local active set
\begin{align*}
\mathcal S_x = \{  1\leq k\leq D  \, : \, \beta_ k (x) \neq 0\} .
\end{align*}

\begin{theorem}\label{th:gradient_ell2}
Suppose that assumptions \ref{cond:density}, \ref{cond:sub_gaussian_inovation}, \ref{cond:lip2} and \ref{cond:lip3} are fulfilled. Let $n\geq 1$ and $k\geq 1$ such that $\overline{\tau} _ k\leq \tau_0$.  Let $\delta\in (0,1)$ and set  $\lambda =  \overline{\tau} _ k  ( \sqrt{ 2   \sigma^2   \log(8D/\delta)/k } + L_2 \overline{\tau} _ k^2 )$. Then, we have with probability larger than $1-\delta$,
\begin{equation}\label{eq:main_bound}
\|  \tilde{\beta}_k  (x) - \beta(x)  \|_2\leq    (24)^2  \sqrt{\#\mathcal S_x }    \left(   \overline \tau_k ^{-1} \sqrt{\frac{ 2   \sigma^2   \log(16D/\delta)}{k} } + L_2 \overline \tau_k  \right),
\end{equation}
as soon as $C_1  ( D\log(  2 n  )  +   \#\mathcal S_x \log( 2 D n / \delta))   \leq k  \leq  C_2  n $,   $  \overline\tau_k   ^{2}     \leq  (   b_f^2 /( C_3 \#\mathcal S_x L ^2 )  \wedge \tau_0 ^2 )$, where $C_1$, $C_2$ and $C_3$ are universal constants.
\end{theorem}

The analysis of the accuracy of the nearest neighbour estimate $\hat m_k(x)$ classically involves the following decomposition of the estimation error
\begin{equation}\label{decomp_bias_var}
 \hat {m}_{k} (x)   - m (x) = \left( \hat {m}_{k}(x)  -   m_{k}(x)   \right) + \left(m_{k}(x)   - m(x)\right),
\end{equation}
where
$ m_{k}(x)  = (1/k)  \sum_{i : X_i \in \mathcal{B}(x, \hat{\tau}_k(x))} m(X_i)$. The approach developed in \citep{jiangNonAsymptoticUniformRates2019} essentially consists in combining this decomposition with the fact that $ \hat{\tau}_k(x)\leq \overline \tau _k $ with large probability. By its own nature, our local linear Lasso regularized estimate of the gradient $\tilde \beta_k$ cannot be treated in the same way. First, in order to take advantage of the Lasso regularization in sparse situations  (\textit{i.e.} when the gradient at $x$ depends on a small number of covariates solely), we rely on a basic inequality \cite[Lemma 11.1]{hastieStatisticalLearningSparsity2015} which is useful when analysing standard Lasso estimates. Second, we need to control the size of the neighbourhoods $\hat{\tau}_k(x)$ on an event of large probability. In this respect, we slightly deviate from the approach of  \citep{jiangNonAsymptoticUniformRates2019}: we do not rely on concentration results over VC classes but only on the Chernoff concentration bound. This way, we can relax significantly the lower bound conditions for $k$ as the dimension $D$ increases, see Theorem 2 in the Supplementary Material (which compares favourably with Corollary 1 in \cite{jiangNonAsymptoticUniformRates2019} for instance).

Balancing between the bias and the variance term of the upper bound provided in \eqref{eq:main_bound} we obtain that the optimal value for $k$ is $k\sim n^{4/(4+D)}$. In this case, the bound stated above yields the rate $n^{-1/(4+D)}$. As a consequence, our bound matches the minimax rate (up to log terms) given in \citep{stoneOptimalGlobalRates1982} for the problem of the estimation of the derivative (in a $L_2$ sense).

\section{Numerical Experiments}\label{sec:exp}

In order to motivate the need for a robust estimator of the gradient, we introduce three different examples of use of our estimator compared to existing approaches. All the code to reproduce the experiments and figures can be found at \url{https://github.com/removed/removed}.

As our estimator is sensitive to the choice of hyperparameters $k$ and $\lambda$ we use a local leave-one-out procedure described in Algorithm~\ref{alg:localcv} for hyperparameter selection. As only the regression variable $Y$ is observed, the regression error is used as a proxy loss in the cross-validation. 
\begin{algorithm}
    \caption{Local Leave-One-Out}\label{alg:localcv}
    \begin{algorithmic}[1] %
        \Require $x$: sample point, $(X, Y)$: training set, $(K, \Lambda)$: grid
        \State{$X_{\text{LoO}} \gets \texttt{Neighbourhood of } x \texttt{ in } X \texttt{ of size } N$}
        \For{$k \in K, \lambda \in \Lambda$}
            \For{$X_i \in X_{\text{LoO}}$}
                \State{$m_i,  \beta_i \gets \texttt{estimated gradient at } X_i \texttt{ w.r.t } X, Y$ \texttt{ using }~\eqref{lll}}
            \EndFor
            \State{$\texttt{error}_{k, \lambda} \gets \frac{1}{N} \sum_{i=1}^N (m_i - Y_i)^2$}
        \EndFor
        \State{$k^\star, \lambda^\star \gets \argmin_{k, \lambda} \texttt{error}_{k, \lambda}$}
        \State{\textbf{return} $k^\star, \lambda^\star$}
    \end{algorithmic}
\end{algorithm}

\subsection{Variable Selection}

While a large number of observations is desirable the same is not necessarily the case for the individual features; a large number of features can be detrimental to the computational performance of most learning methods but also harmful to the actual performance. In order to mitigate the detrimental impact of the high dimensionality, or \emph{curse of dimensionality}, one can try to reduce the effective dimension of the problem. A large body of work exists on dimensionality reduction as a preprocessing step that considers the intrinsic dimensionality of $X$ by considering for example that $X$ lies on a lower-dimensional manifold. Those approaches only consider $X$ in isolation and do not take into account $Y$ which is the variable of interest. It is possible to use the information in $Y$ to direct the dimension reduction of $X$, either by treating $Y$ as side information, as is done in \cite{bachPredictiveLowrankDecomposition2005}, or by considering the existence of an explicit \emph{index space} such that $Y_i = g(v_1^\intercal X_i, \cdots, v_m^\intercal X_i) + \varepsilon_i$ as is done in~\cite{dalalyanNewAlgorithmEstimating2008}. In the latter case, it is possible to observe that the $\emph{index space}$ lies on the subspace spanned by the gradient.

In contrast with the work of~\cite{dalalyanNewAlgorithmEstimating2008} our approach is local and it is therefore possible to retrieve a different subspace in different regions of $\mathbb{R}^D$. As localizing the estimator increases its variance, we choose to only identify the dimensions of interest instead of estimating the full projection matrix.
We introduce Algorithm~\ref{alg:LocalLinearTree} to exploit the local aspect of our estimator in order to direct the cuts in a random tree: at each step, cuts are drawn randomly with probability proportional to estimated mean absolute gradient in the cell.
\begin{algorithm}
    \caption{Node Splitting for Gradient Guided Trees}\label{alg:LocalLinearTree}
    \begin{algorithmic}[1] %
        \Require $(X, Y)$: training set, $\texttt{Node}$: indexes of points in the node
        \State{$\nabla m (X_i) \gets \texttt{estimated gradient at } X_i, \, \forall i \in \texttt{Node}$ \texttt{ using }~\eqref{lll}}
        \State{$\omega \gets \sum_{i \in \texttt{Node}} \lvert \nabla m(X_i) \rvert$}
        \State{$K \gets \texttt{sample } \sqrt{D} \texttt{ dimensions in } \{1, \ldots, d\} \texttt{ with probability weights} \propto \omega$}
        \State{$k, c \gets \texttt{best threshold } c \texttt{ and dimension } k$}
        \State{$\textbf{return } k, c$}
    \end{algorithmic}
\end{algorithm}
We demonstrate the improvements brought by guiding the cuts by the local information provided by the gradient by comparing the performance of a vanilla regression random forest with the same procedure but with local gradient information. 
We consider five datasets: the Breast Cancer Wisconsin (Diagnostic) Data Set introduced in~\cite{streetNuclearFeatureExtraction1993}; the Heart Disease dataset introduced by~\cite{detranoInternationalApplicationNew1989}; the classic Diamonds Price dataset; the Gasoline NIR dataset introduced by \citep{kalivasTwoDataSets1997} and the Sloan Digital Sky Survey DR14 dataset of~\cite{abolfathiFourteenthDataRelease2018}.

As seen in Table~\ref{table:results}, gradient guided split sampling consistently outperform the vanilla variant. When all variables are relevant, as is the case when the variables were carefully selected by the practitioner with prior knowledge, our variant performs similarly to the original algorithm while performance is greatly improved when only a few variables are relevant, such as in the NIR dataset \citep{portierBootstrapTestingRank2014}.
\begin{table}
    \centering
    \begin{tabular}{lrrrr}
        \toprule
        & \multicolumn{2}{c}{Description} & \multicolumn{2}{c}{Loss} \\
        \cmidrule(l){2-3} \cmidrule(l){4-5} \\
        Dataset & $n$ & $D$ & Random Forest & Gradient Guided Forest \\
        \midrule
        Wisconsin & $569$ & $30$ & $0.0352 \pm 3.29\times10^{-4}$ & $\mathbf{0.0345} \pm 3.35\times10^{-4}$ \\
        Heart Disease & $303$ & $13$ & $0.128 \pm 6.6\times10^{-4}$ & $\mathbf{0.124} \pm 8.6\times10^{-4}$ \\
        Diamonds & $53940$ & $23$ & $680033 \pm 3.45\times10^{9}$ & $\mathbf{664265} \pm 2.81\times10^{9}$ \\
        Gasoline NIR & $60$ & $401$ & $0.678 \pm 0.451$ & $\mathbf{0.512} \pm 0.347$ \\
        SDSS & $10000$ & $8$ & $0.872\times10^{-3} \pm 4.50\times10^{-6}$ & $\mathbf{0.776}\times10^{-3} \pm 6.00\times10^{-6}$ \\
        \bottomrule
    \end{tabular}
    \caption{Performance of the two random forest variants}\label{table:results}
\end{table}

\subsection{Gradient Free Optimization}

Many of the recent advances in the field of machine learning have been made possible in one way or another by advances in optimization; both in how well we are able to optimize complex function and what type of functions we are able to optimize if only locally. Recent advances in automatic differentiation as well as advances that push the notion of \emph{what} can be differentiated have given rise to the notion of \emph{differentiable programming} \citep{innesDifferentiableProgrammingSystem2019} in which a significant body of work can be expressed as the solution to a minimization problem usually then solved by gradient descent.

We study here the use of the local linear estimator of the gradient in Algorithm~\ref{alg:lolamin} in cases where analytic or automatic differentiation is impossible, and compare it to a standard gradient free optimization technique as well as the oracle where the gradients are known.

We minimize the standard but challenging Rosenbrock function:
\begin{align}
    f(x) = 100 \sum_{i=1}^{d-1} (x_{i+1} - x_i)^2 + (x_i - 1)^2.
\end{align}

\begin{algorithm}
    \caption{Estimated Gradient Descent}\label{alg:lolamin}
    \begin{algorithmic}[1] %
        \Require $x_0$: initial guess, $f$: function $\mathbb{R}^D \to \mathbb{R}$, $M$: budget
        \State{$X \gets X_1, \ldots, X_M \texttt{ with } X_i \sim \mathcal{N}(x_0, \varepsilon \times I_D)$}
        \State{$Y \gets f(X) := f(X_1), \ldots, f(X_M)$}
        \While{\texttt{not StoppingCondition}}
            \State{$m, \Delta \gets \texttt{estimated gradient at } x \texttt{ w.r.t } X, Y$ \texttt{ using }~\eqref{lll}}
            \State{$X \gets X, X_1, \ldots, X_M \texttt{ with } X_i \sim \mathcal{N}(\texttt{GradientStep}(x, \Delta), \varepsilon \times I_D)$}
            \State{$Y \gets f(X)$}
            \State{$x \gets \argmin_{X_i} \{ f(X_i) \}$}
        \EndWhile
        \State{$\textbf{return } x$}
    \end{algorithmic}
\end{algorithm}

We apply the previous method to the minimization of the log-likelihood of a logistic model on the UCI's Adult data set, consisting of $48842$ observations and $14$ attibutes amounting to $101$ dimensions once one-hot encoded and an intercept added.
\begin{align*}
    \mathcal{L}_\theta (X) =  - \sum_i Y_i \log (1 + \exp (-\theta X_i)) -  (1 - Y_i) \log (1 + \exp (\theta X_i)), \, \theta \in \mathbb{R}^{101}.
\end{align*}
\begin{figure}[hbt!]
    \centering
    \begin{minipage}{0.66\textwidth}
        \begin{subfigure}[T]{0.5\linewidth}
            \centering
            \input{figs/rosenbrock_50.tikz} %
        \end{subfigure}
        \begin{subfigure}[T]{.5\linewidth}
            \centering
            \input{figs/rosenbrock_100.tikz} %
        \end{subfigure}
        \caption{Rosenbrock function}
    \end{minipage}
    \begin{minipage}{0.33\textwidth}
        \vspace{0.3cm}
        \begin{subfigure}[T]{1.\linewidth}
            \hspace{-0.5cm}
            \centering
            \input{figs/logisticreg.tikz} %
        \end{subfigure}
        \caption{Logistic Regression}
    \end{minipage}
\end{figure}

\subsection{Disentanglement}

\emph{Disentangled Representation Learning} aims to learn a representation of the input space such that the independent dimensions of the representation each encode separate but meaningful attributes of the original feature space.
We show here how our estimator can be useful for retrieving the dimensions associated with a concept in a supervised manner.

A $\beta$-VAE \citep{higginsBetaVAELearningBasic2017} model is trained on the \texttt{CACD2000} dataset of celebrity faces with age labels to first build low-dimensional representations of the images and then extract the direction relating to age. We learn $\mathcal{E}_{\phi}$ and $\mathcal{D}_{\theta}$ parameterizing $q_\phi$ and $p_\theta$, to minimize the loss
\begin{equation}
    \mathcal{L} (\theta, \phi; x, z, \beta) = \mathbb{E}_{q_\phi (z \mid x)} \left[ \log p_\theta (x \mid z) \right] - \beta \infdiv{q_\phi x}{x},
\end{equation}
where $\beta$ acts as a constraint on the representational power of the latent distribution; $\beta = 1$ leads to the standard VAE formulation of \cite{kingmaAutoEncodingVariationalBayes2014} while $\beta > 1$ increases the level of disentanglement.
We learn a $512$-dimensional representation of the $128\times 128$ images and encode all the \texttt{CACD2000} images (see Appendix for details).

Using our estimator it is possible to estimate the gradient $\nabla m$ of $\mathbb{E} [ Y \mid Z = z ]$ with respect to the latent variable $Z$ (illustrated in the Appendix). It is then possible to analyse the sparsity of $\nabla m$ to quantify the quality of the disentanglement for varying level of $\beta$ by quantifying how far from a single dimension the gradient for the age is concentrated. As the true dimension is unknown, we instead measure the angular distance to all dimensions reweighted by the magnitudes of the partial derivatives:
\begin{align}
    \sum_i \frac{\lvert\hat \nabla_i m(x) \rvert}{\lvert \hat \nabla m(x) \rvert} \cos (e_i, \frac{1}{n} \sum_k \lvert \hat \nabla m(x) \rvert), \, \text{where} \, \cos (a, b) = \frac{a \cdot b}{\lVert a \rVert \lVert b \rVert}.
\end{align}
We observe in Figure~\ref{fig:disentangle} that as $\beta$ increases the age slowly become disentangled, as expected if one considers the age to be an important and independent characteristic of human faces.
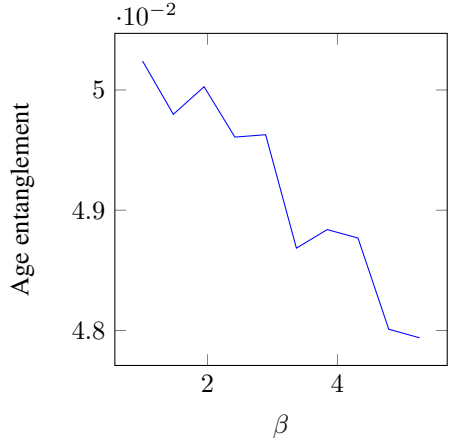
\begin{figure}[h]
    \centering
    \begin{tikzpicture}
\begin{axis}[xlabel={$\beta$}, ylabel={Age entanglement},width=6cm,height=6cm]
    \addplot+[no marks]
        table[row sep={\\}]
        {
            x  y  \\
            1.0  0.05024031549692154  \\
            1.4736842  0.04979776591062546  \\
            1.9473684  0.05002676323056221  \\
            2.4210527  0.04960823431611061  \\
            2.8947368  0.049627795815467834  \\
            3.368421  0.04868478327989578  \\
            3.8421052  0.04883917048573494  \\
            4.3157897  0.048768892884254456  \\
            4.7894735  0.0480106882750988  \\
            5.263158  0.04793939366936684  \\
        }
        ;
\end{axis}
\end{tikzpicture}
    \caption{Quality of disentanglement with respect to the age} \label{fig:disentangle}
\end{figure}

While not an entirely adequate metric for disentanglement, not only because disentanglement does not necessarily require the dimensions to be the one an observer expected but more importantly because this metric requires an annotated dataset; we believe this metric can be useful for practitioners. By measuring how close the estimated gradients are to the axis, with respect to an annotated dataset of characteristics of interest, a practitioner can ensure his model is sufficiently disentangled for downstream tasks such as face manipulation by a user. We also believe it is possible to design an end-to-end differentiable framework in order to force disentanglement to consider the characteristics of interest: our estimator is the solution to a convex optimization program and as such admits an adjoint and it is therefore possible to fit a local linear estimator inside an automatic differentiation framework such as done in \cite{agrawalDifferentiableConvexOptimization2019}.
\section{Conclusion}\label{sec:conclusion}
In this paper, we have studied the estimator of the gradient of the (supposedly sparse) regression function obtained by solving a regularized local linear version of the $k$-NN problem with a $\ell_1$ penalty. Nonasymptotic bounds for the local estimation error have been established, improving upon those obtained for alternative methods in sparse situations.
Beyond its theoretical properties and its computational simplicity, the local estimation method promoted here is shown to be the key ingredient for designing efficient algorithms for variable selection and $M$-estimation, as supported by various numerical experiments. Hopefully,  this work shall pave the way to the elaboration of novel statistical learning procedures that exploits the local structure of the gradient, and that the theory will be extended to take into account the underlying geometry of the space in order to obtain a fast convergence rate depending on the true intrinsic dimension instead of the ambient dimension.
\newpage
\section*{Broader Impact}\label{sec:impact}
In this work we show both theoretically and empirically that is possible to derive estimates of the gradient of an unknown function even in a high-dimensional and low sample regime. More importantly we show through several experiments that these estimates are sufficiently accurate to be of use in downstream tasks such variable selection and optimization. We believe those encouraging results should empower practitioners to make use of rough gradient estimates when they have reasons to believe such a gradient exists, especially in cases where the direction is more important than the exact value of the gradient.

The theoretical results presented in this paper do not present any foreseeable societal consequence or ethical problems.

\bibliographystyle{abbrvnat6}
\bibliography{Biblio}

\section{Technical Proofs}

\subsection{Additional Result - Pointwise Nearest Neighbour Estimation of the Regression Function}

Though it concerns the local estimation error, the bound  in the theorem below can be viewed as a refinement of the nonasymptotic results recently established in \cite{jiangNonAsymptoticUniformRates2019} (see also \cite{kpotufeKNNRegressionAdapts2011}), which provide uniform bounds in $x$. It requires a local smoothness condition for the regression function. Throughout this subsection, $\vert\vert .\vert\vert$ denotes any norm on $\mathbb{R}^D$.

\begin{assumption}\label{cond:lip}
The regression function $m(z)$ is $L_1$-Lipschitz at $x$, \textit{i.e.} there exists $L_1>0$ such that for all $z \in \mathcal{B}(x, \tau_0)=\{x'\in \mathbb{R}^D:\; \vert\vert x'-x\vert\vert \leq \tau_0  \}$,
\begin{equation*}
|m (x) - m (z) | \leq  L_1\|x-z\| .
\end{equation*}
\end{assumption}

\begin{theorem}\label{theorem:loco}
Suppose that assumptions \ref{cond:density}, \ref{cond:sub_gaussian_inovation} and \ref{cond:lip} are fulfilled and that $2  k  \leq n \tau_0  b_fV_D $. Then for any $\delta \in (0,1)$ such that  $ k\geq 4   \log(2n/\delta)   $, we have with probability $1-\delta$:
\begin{align*}
|\hat m_k (x) - m  (x)|  \leq  \sqrt{ \frac{2  \sigma^2 \log(4/\delta)}{k} } + L_1 \left (\frac{ 2  k }{ n b_fV_D}  \right)^{1/ D},
\end{align*}
where $V_D= \int \mathds{1}_{\{x\in \mathcal{B}(0,1)\}}dx $ denotes the volume of the unit ball.
\end{theorem}

We obtain a weaker condition on the value of $k$ than that obtained in \cite{jiangNonAsymptoticUniformRates2019} (see  Corollary 1 therein), due to our different treatment of the approximation term (the second term in decomposition \eqref{decomp_bias_var}) is different (see the argument detailed in the Supplementary Material). With $k\sim n^{2/(2+D)}$, the bound stated above yields the minimax rate $n^{-1/(D+2)}$.

\subsection{Auxiliary Results}
As a first go, we recall or prove various auxiliary results that are involved in the proof of Theorem \ref{th:gradient_ell2}, and in that of Theorem \ref{theorem:loco} as well.
\medskip

The following inequality follows from the well-known Chernoff bound, see \textit{e.g.} \citep{boucheronConcentrationInequalitiesNonasymptotic2013}.

\begin{lemma}\label{lemma=chernoff}
Let $(Z_i)_{i\geq 1}$ be a sequence of i.i.d. random variables valued in $\{0,1\}$. Set $\mu =  n \mathbb E [Z_1]$ and $S = \sum_{i=1} ^n Z_i $.   For any $\delta \in (0,1)$ and all $n\geq 1$, we have with probability $1-\delta$:
\begin{align*}
S \geq \left(1- \sqrt{ \frac{2 \log(1/\delta)  }{  \mu} } \right) \mu  .
\end{align*}
 In addition, for any $\delta \in (0,1)$ and $n\geq 1$, we have with probability $1-\delta$:
\begin{align*}
S \leq \left(1 +  \sqrt{ \frac{3 \log(1/\delta)   }{  \mu} }  \right) \mu  .
\end{align*}
\end{lemma}

\begin{proof}
Using the Chernoff lower tail \citep{boucheronConcentrationInequalitiesNonasymptotic2013}, for any $t > 0 $ and $n\geq 1$, it holds that
\begin{align*}
\mathbb P \left( S < (1-t ) \mu \right) \leq \left( \frac{\exp(-t) }{ (1- t) ^{1-t} }  \right)^\mu.
\end{align*}
Because for any $t\in (0,1)$, ${\exp(-t) } / { (1- t) ^{1-t} } \leq  \exp( - {t^2  }/{2} ) $, we obtain that for any $t > 0 $ and $n\geq 1$,
\begin{align*}
\mathbb P \left( S < (1-t ) \mu \right) \leq \exp\left( - \frac{t^2 \mu }{2} \right) ,
\end{align*}
the bound being obvious when $t\geq 1$. In the previous bound ,choose $t = \sqrt{ {2 \log(1/\delta)  } /{  \mu}}$ to get the stated inequality. The second inequality is obtained by inverting the Chernoff upper tail:
\begin{align*}
\mathbb P \left( S > (1+t ) \mu \right) \leq \left( \frac{\exp(t) }{ (1+ t) ^{1+ t} }  \right)^\mu.
\end{align*}

\end{proof}

The following inequality is a well-known concentration inequality for sub-Gaussian random variables, see \textit{e.g.} \citep{boucheronConcentrationInequalitiesNonasymptotic2013}.

\begin{lemma}\label{lemma:SG}
Suppose that $Z$ is sub-Gaussian with parameter $s^2>0$, i.e., $Z$ is real-valued, centered and for all $\lambda>0$, $\mathbb E [ \exp(\lambda Z)  ] \leq \mathbb E [ \exp( \lambda^2 s^2 / 2)]$, then with probability $1-\delta$,
\begin{align*}
|Z|  \leq \sqrt {  2s ^2 \log( 2 / \delta) } .  
\end{align*} 
\end{lemma}

We shall also need a concentration inequality tailored to Vapnik-Chervonenkis (VC) classes of functions. The result stated in Lemma \ref{vc_bound} below is mainly a consequence of the work in \cite{gineConsistencyKernelDensity2001}. Our formulation is slightly different, the role played by the VC constants ($v$ and $A$ below) being clearly quantified. 

Let $\mathcal F$ be a bounded class of measurable functions defined on $\mathcal X$. Let $U$ be a uniform bound for the class $\mathcal F$, i.e. $|f(x)| \leq U$ for all $f\in \mathcal F$ and $x\in \mathcal X$. The class $\mathcal F$ is called VC with parameters $(v ,A)$ and uniform bound $U$ if
\begin{align*}
\sup_Q\mathcal N  \left( \epsilon U , \mathcal F , L_2 (Q)  \right)\leq  \left( \frac A \epsilon \right)  ^{v }  ,
\end{align*}
where $\mathcal N (.,\mathcal{F},L_2(Q))$ denotes the covering numbers of the class $\mathcal{F}$ relative to $L_2(Q)$, see \textit{e.g.} \citep{vaartWeakConvergenceEmpirical1996}.  For notational simplicity and with no loss of generality, we require in the definition of a VC class that $A\geq 3\sqrt e$ and $v\geq 1$. Define 
 $\sigma^2 \geq \sup_{f\in \mathcal F} \Var(f (X_1) )$. We shall work with the conditions
\begin{align}
\label{eq:constant_n_delta_1}  &\sqrt n \sigma \geq c_1  \sqrt {U^2 v\log(AU / \sigma ) },\\
\label{eq:constant_n_delta_2}   &\sqrt n \sigma  \geq c_2 \sqrt { U^2  \log( 2 / \delta)} ,
\end{align} 
where the constant $c_1$ and $c_2$ are specified in the following statement.
\begin{lemma}\label{vc_bound}
Let $\mathcal F$ be a VC class of functions with parameters $(v,A)$ and uniform bound $U>0$ such that $\sigma \leq U$.  Let $n\geq 1$ and $\delta\in (0,1) $. There are three positive universal constants $c_1$, $c_2$ and $c_3$ such that, under conditions \eqref{eq:constant_n_delta_1}  and \eqref{eq:constant_n_delta_2}, we have with probability $1-\delta$,
\begin{align*}
&  \sup_{f\in \mathcal F} \left| \sum_{i=1} ^n \{ f(X_i) - \mathbb E f(X_1) \}\right|   \leq c_3 \sqrt { n\sigma^2 v \log( A U / (\sigma \delta) ) }   .
  \end{align*} 
\end{lemma}

\begin{proof}
Set $\Lambda = v\log(AU / \sigma )$. Using \cite[equation (2.5) and (2.6)]{gineConsistencyKernelDensity2001}, we get
\begin{align*}
\mathbb E\left [\sup_{f\in \mathcal F} \left| \sum_{i=1} ^n \{ f(X_i) - \mathbb E f(X_1) \}\right|\right] \leq  C\left( U \Lambda  +   \sqrt { n\sigma^2 \Lambda} \right) \leq 2 C \sqrt { n\sigma^2 \Lambda} , &\\
\mathbb E\left [\sup_{f\in \mathcal F} \left| \sum_{i=1} ^n \{ f(X_i) - \mathbb E f(X_1) \}^2\right|\right] \leq  \left( \sqrt n \sigma  +   KU  \sqrt {  \Lambda} \right)^2 \leq 4  n \sigma^2 := V&,
\end{align*}
where $C>0$ and $K>0$ are two universal constants. The inequalities below are both obtained by taking $c_1$ large enough. Let
\begin{align*}
&Z = \sup_{f\in \mathcal F} \left| \sum_{i=1} ^n \{ f(X_i) - \mathbb E f(X_1) \}\right|
\end{align*}
We recall Talagrand's inequality \citep[Theorem 1.4]{talagrandNewConcentrationInequalities1996} (or \cite[equation (2.7)]{gineConsistencyKernelDensity2001}), for all $t>0$,
\begin{align*}
\mathbb P \left(| Z - \mathbb E Z | >t  \right )\leq K' \exp\left( - \frac{t}{2K'U} \log( 1+ 2tU / V)   \right),
\end{align*}
where $K'>1$ is a universal constant. Using the fact that for all $t\geq 0$, $  t / (2+2t/3) \leq \log( 1+ t )$, we get
\begin{align*}
\mathbb P \left(| Z - \mathbb E Z | >t  \right )\leq K' \exp\left( -  \frac{t^2  }{2K'(V+ 2tU / 3) }   \right).
\end{align*}
Inverting the bound we find that for any $\delta\in (0,1)$, with probability $1-\delta$,
\begin{align*}
 | Z - \mathbb E Z | &\leq  \sqrt{ 2 K' V \log( K' / \delta)}  +  (4 K' U / 3) \log( K' / \delta) \\
 &\leq   \sqrt{ 2 K' V K''  \log( 2 / \delta)}  +  (4 K' U / 3)K'' \log( 2 / \delta)
  \end{align*} 
for some $K''>0$. Taking $c_2$ large enough, we ensure that $ 2V = 8 n \sigma^2 \geq (4U / 3)^2 K' K''\log(2 / \delta) $. Then using the previous bound on the expectation, it follows that with probability $1-\delta$,
\begin{align*}
 | Z|  & \leq  2 C \sqrt { n\sigma^2 \Lambda} + 2 \sqrt{ 8   n \sigma^2 K'K'' \log( 2 / \delta)}  \\
 & = 2C ( 1 + 8 K' K'')  \sqrt { n\sigma^2} \left( \Lambda + \sqrt{ \log( 2 / \delta)) } \right).
\end{align*} 
We then conclude by using the bound$ \sqrt a + \sqrt b \leq \sqrt 2 \sqrt {a+b} $.
\end{proof}

\subsection{Intermediary Results}
We now prove some intermediary results used in the core of the proof of the main results.

Define
\begin{align*}
\overline{\tau}_{k}  = \left(\frac{ 2  k }{ n b_fV_D}  \right)^{1/ D}.
\end{align*}

\begin{proposition}\label{prop:tau}
Suppose that Assumption \ref{cond:density} is fulfilled and that $\overline \tau_{k} \leq \tau_0$. Then for any $\delta \in (0,1)$ such that  $ k\geq 4   \log(n/\delta)   $, we have with probability $1-\delta$:
\begin{align*}
 \hat \tau_k(x)   \leq \overline \tau _k .
\end{align*}
\end{proposition}

\begin{proof}
 Using Assumption \ref{cond:density} yields
\begin{align*}
\mathbb P (X\in \mathcal{B} (x, \overline \tau_{k} ) )  =   \int_{\mathcal{B} (x, \overline \tau_{k}  )} f_X  \geq b_f \int _{\mathcal{B} (x, \overline \tau_{k}  ) } d\lambda = b_fV_{D} \overline \tau_{k} ^{  D}  = 2k/n.
\end{align*}
Consider the set formed by the $n$ balls $ \mathcal{B} (x, \overline \tau_{k}   ) $, $  1\leq k\leq n $. From Lemma \ref{lemma=chernoff} with $Z_i = \mathds 1 _{ \mathcal{B} (x, \overline \tau_{k}  )} (X_i )$, $\mu \geq 2k$, and the union bound, we obtain that for all $\delta\in (0,1)$ and any $k= 1,\ldots, n$:
\begin{align*}
\sum_{i=1} ^n \mathds 1 _{ \mathcal{B} (x, \overline \tau_{k}  ) } (X_i ) &\geq \left(1- \sqrt{ \frac{2 \log(n/\delta)  }{ 2k} } \right) 2k  .
\end{align*}
As $ k\geq 4   \log(n/\delta)  $, it follows that
\begin{align*}
\sum_{i=1} ^n \mathds 1 _{ \mathcal{B} (x, \overline \tau_{k}  ) } (X_i ) & \geq k - (  \sqrt{ 4 k  \log(n/\delta)} -  k)\geq k.
\end{align*}
Hence $ \mathbb P_n ( \mathcal{B} (x, \overline \tau_{k} ) ) \geq k / n $, denoting by $\mathbb P_n $ the empirical distribution of the $X_i$'s. By definition of $\hat \tau _{k}(x)$ it holds that  $\hat \tau _{k}(x) \leq  \overline \tau _{k}(x)$.
\end{proof}

Define
\begin{align*}
\underline \tau_{k} = \left (\frac{   k  }{ 2n U_f V_D}  \right)^{1/ D}.
\end{align*}

\begin{proposition}\label{prop:tau2}
Suppose that Assumption \ref{cond:density} is fulfilled and that $\underline \tau_{k}  \leq \tau_0$. Then for any $\delta \in (0,1)$ such that  $ k\geq 4   \log(n/\delta)   $, we have with probability $1-\delta$:
\begin{align*}
 \hat \tau_k  \geq \underline \tau_{k}  .
\end{align*}
\end{proposition}

\begin{proof}
 Using Assumption \ref{cond:density} yields
\begin{align*}
\mathbb P (X\in \mathcal{B} (x, \underline \tau_{k}  ) )  =   \int_{\mathcal{B} (x, \underline  \tau_{k}  )} f_X  \leq U_f \int _{\mathcal{B} (x, \underline  \tau_{k}  ) } d\lambda = U_fV_{D}  \underline  \tau_{k} ^{D}  = k/(2n).
\end{align*}
Consider the set formed by the $n$ balls $ \mathcal{B} (x, \underline  \tau_{k} ) $, $  1\leq k\leq n $. From Lemma \ref{lemma=chernoff} with $Z_i = \mathds 1 _{ \mathcal{B} (x, \underline  \tau_{k}  )} (X_i )$, $\mu \leq k/2$, and the union bound, we obtain that for all $\delta\in (0,1)$ and $k= 1,\ldots, n$
\begin{align*}
\sum_{i=1} ^n \mathds 1 _{ \mathcal{B} (x,  \underline  \tau_{k}  ) } (X_i ) &\leq \left(1+ \sqrt{ \frac{6 \log(n/\delta)  }{ k} } \right) k/2  .
\end{align*}
Using that $ k\geq 6   \log(n/\delta)  $, it follows that
\begin{align*}
\sum_{i=1} ^n \mathds 1 _{ \mathcal{B} (x,  \underline  \tau_{k}  ) } (X_i ) & \leq k +  (  \sqrt{ (6/4) k  \log(n/\delta)} - k/2)\leq k.
\end{align*}
Hence $ \mathbb P_n ( \mathcal{B} (x, \underline  \tau_{k}  ) ) \leq k / n $. By definition of $\hat \tau _{n}(k)(x)$ it holds that  $\underline  \tau _{k}  \leq     \hat \tau _{k} (x)$.
\end{proof}

\begin{proposition}\label{prop:var1}
    Suppose that Assumption \ref{cond:sub_gaussian_inovation} is fulfilled. Then for any $\delta \in (0,1)$, we have with probability $1-\delta$:
    \begin{align*}
        \left| \sum_{i=1} ^n \xi_i \mathds 1 _{ \mathcal{B} (x, \hat \tau_{k} (x) ) } (X_i ) \right|  \leq \sqrt{2 k   \sigma^2 \log(2/\delta)} .
    \end{align*}
\end{proposition}
\begin{proof}
 Set $w_i = \mathds 1 _{ \mathcal{B} (x, \hat  \tau_{k} (x) ) } (X_i )$. Note that $\sum_{i=1} ^n w_i^2 = k$ almost surely. The result follows from the application of Lemma \ref{lemma:SG} to the random variable $ \sum_{i=1} ^n \xi_i w_i$, which is sub-Gaussian with parameter $k \sigma^2$. To check this, it is enough to write
    \begin{align*}
        \mathbb E\left[\exp\left( \lambda  \sum_{i=1} ^n \xi_i w_i \right)\right] &=\mathbb E\left[\mathbb E \left[ \exp\left( \lambda  \sum_{i=1} ^n \xi_i w_i \right)\mid X_1,\ldots X_n\right]\right]\\
        &=\mathbb E\left[ \prod_{i=1} ^n \mathbb E \left[ \exp\left( \lambda \xi_i w_i\right)\mid X_1,\ldots X_n\right]\right]\\
        & \leq \mathbb E\left[ \prod_{i=1} ^n \mathbb E \left[ \exp\left( \lambda^2  \sigma^2 w_i^2 /2 \right)\mid X_1,\ldots X_n\right]\right]\\
        &= \mathbb E\left[   \exp\left( \lambda^2  \sigma^2 \sum_{i=1} ^n w_i^2 /2 \right) \right]=   \exp\left( \lambda^2  \sigma^2 k /2 \right) .
    \end{align*}
\end{proof}

\begin{proposition}\label{prop:var2}
    Suppose that Assumption \ref{cond:density}  and \ref{cond:sub_gaussian_inovation} are fulfilled and that $\overline \tau_{k} \leq \tau_0$. Let $\hat h_i:= h_i(X_1,\ldots,X_n) $ such that $a_k = \sup_{i:X_i\in \mathcal B (x, \overline \tau_k) } |\hat h_i |$. Then for any $\delta \in (0,1)$ such that $ k\geq 4 \log(2n/\delta)   $, we have with probability $1-\delta$: 
        \begin{align*}
        \left| \sum_{i=1} ^n \xi_i \hat h_i \mathds 1 _{ \mathcal{B} (x, \hat \tau_{k} (x) ) } (X_i ) \right|  \leq \sqrt{2 k\sigma^2 a_k^2  \log(4/\delta)} .
    \end{align*}
\end{proposition}

\begin{proof}
 Set $w_i = \mathds 1 _{ \mathcal{B} (x, \hat  \tau_{k} (x) ) } (X_i )$. Note that $\sum_{i=1} ^n w_i^2 = k$ almost surely. The result follows from the fact that conditioned upon $X_1,\ldots, X_n$, the random variable $ \sum_{i=1} ^n \xi_i h_i w_i$ is sub-Gaussian with parameter $\sigma^2 k \hat a_k^2 $ with $\hat a_k = \sup_{i:X_i\in \mathcal B (x, \overline \tau_k ) } | \hat h_i|$. To check this, it suffices to write
     \begin{align*}
       \mathbb E \left[ \exp\left( \lambda  \sum_{i=1} ^n \xi_i \hat h_i  w_i \right)\mid X_1,\ldots X_n\right]       &=\  \prod_{i=1} ^n \mathbb E \left[ \exp\left( \lambda \xi_i \hat h_i  w_i\right)\mid X_1,\ldots X_n\right]\\
        & \leq  \prod_{i=1} ^n   \exp\left(   \lambda^2  \sigma^2 \hat h_i^2 w_i^2 /2 \right) \\
        &=   \exp\left( \lambda^2  \sigma^2 \sum_{i=1} ^n \hat h_i^2 w_i /2 \right)  \leq \exp\left(  \lambda^2  \sigma^2 k \hat a_k^2 /2 \right).
    \end{align*}
Then, for any $t>0$,
\begin{align*}
\mathbb P \left( \left| \sum_{i=1} ^n \xi_i h_i w_i  \right| > t  \right) &\leq \mathbb P \left( \left| \sum_{i=1} ^n \xi_i h_i w_i  \right| > t  ,\, \hat \tau_k(x)  \leq \tau _k (x)\right) + \mathbb P ( \hat \tau_k(x)  \leq \tau _k (x) ) \\
&\leq  \mathbb E \left[   \mathbb P \left(   \left| \sum_{i=1} ^n \xi_i h_i w_i  \right| > t\mid X_1,\ldots, X_n  \right)  \mathds{1}_{\{ \hat \tau_k(x)  \leq \tau _k (x)\}} \right] + \mathbb P ( \hat \tau_k(x)  \leq \tau _k (x) ) \\
&\leq \mathbb E \left[   2\exp( -t^2 / (2k \sigma^2  \hat a_k^2 ) )  \mathds{1}_{\{ \hat \tau_k(x)  \leq \tau _k (x)\}} \right] + \mathbb P ( \hat \tau_k(x)  \leq \tau _k (x) )\\
&\leq      2\exp( -t^2 / (2k \sigma^2   a_k^2 ) )  + \mathbb P ( \hat \tau_k(x)  \leq \tau _k (x) )
\end{align*}
We obtain the result by choosing $t = \sqrt { 2k \sigma^2   a_k^2  \log( 4 / \delta )} $ and applying Proposition \ref{prop:tau} (to obtain that $\mathbb P ( \hat \tau_k(x)  \leq \tau _k (x) )\leq \delta/2$).
\end{proof}

\begin{proposition}\label{prop:cov_bound}
Suppose that Assumption \ref{cond:density} and \ref{cond:lip3} is fulfilled. Let $\tau >0$, $n\geq 1$, and $\delta\in (0,1)$ such that $ \tau \leq \tau_ 0 $ and   $24 n U_f (2\tau ) ^{D}  \geq   \log(2D^2 /\delta )  $, then with probability $1-\delta$,
\begin{align*}
\max_{1\leq j ,j'\leq D}\left| \sum_{i=1} ^n \left\{  (X_{i,j} - x )(X_{i,j'} - x )^T \mathds{1} _{\mathcal B ( x , \tau) }  (X_i)  -   \mathbb E [  (X_{1,j} - x )(X_{1,j'} - x )^T \mathds{1} _{\mathcal B ( x , \tau) }  (X_1) ] \right\}  \right|& \\
\leq  (2\tau ) ^{2 } \sqrt {  \frac{2 U_f n  (2\tau ) ^{D  }}{3}   \log(2D^2/\delta) } .& 
\end{align*}

\end{proposition}

\begin{proof}
We use Bernstein inequality: for any collection $(Z_1,\ldots, Z_n)$ of independent zero-mean random variables such that for all $i=1,\ldots, n$, $|Z_i| \leq m$ and $\mathbb E Z_i^2 \leq v $, it holds that with probability $1-\delta$,
\begin{align*}
\left| \sum_{i=1} ^n  Z_i  \right| \leq  \sqrt {2 n v\log(2/\delta) } + (m/3) \log(2/\delta) .
\end{align*}  
Applying this with 
\begin{align*}
&W_ i  =  \frac{(X_{i,j} -x) }{2\tau}  \frac{(X_{i,j'} -x)}{2\tau} \mathds{1} _{\mathcal B (0,\tau) }  (X_i) , \\
&Z_i = W_i - \mathbb E [W_i],
\end{align*}
we can use
\begin{align*}
|Z_i| \leq 2| W_i| \leq  1/4 = m,
\end{align*}
and 
\begin{align*}
\mathbb E [ (W_i- \mathbb E W_i  )^2]& \leq \mathbb E [W_i^2]  = \mathbb E \left[  \left| \frac{(X_{i,j} -x)}{2\tau} \frac{(X_{i,j'} -x)}{2\tau} \right|^2 \mathds{1} _{\mathcal B (0,\tau) }(X_i) \right]\\
& = \int   \left|  \frac{(y_{j} -x)}{2\tau} \frac{( y_{j'} -x)}{2\tau} \right|^2 \mathds{1} _{\mathcal B (0,\tau) }(y)  f( y ) dy \\
& \leq  U_f  \int   \left|  \frac{(y_{j} -x)}{2\tau} \frac{( y_{j'} -x)}{2\tau} \right|^2 \mathds{1} _{\mathcal B (0,\tau) }(y)  dy \\
& = U_f(2\tau ) ^{D}  \int   \left| u_j  u_{j'}  \right|^2 \mathds{1} _{\mathcal B (0,1/2) }(u) du\\
& \leq U_f(2\tau ) ^{D}  \int    (u_j ^2 +   u_{j'}^2 )/2  \mathds{1} _{\mathcal B (0,1/2) }(u) du\\
&  =  U_f(2\tau ) ^{D}  \int    u_1 ^2   \mathds{1}_{\mathcal B (0,1/2) }(u) du\\
& = U_f(2\tau ) ^{D}  \int_{[-1/2,1/2]  }    u_1^2   du_1   =  \frac{ U_f (2\tau ) ^{D}}{12}  = v.
\end{align*}
We have shown that, with probability $1-\delta$,
\begin{align*}
\left| \sum_{i=1} ^n  Z_i  \right| \leq  \sqrt {  \frac{ nU_f  (2\tau ) ^{D}}{6}  \log(2/\delta) } + (1/12) \log(2/\delta) .
\end{align*}
Because $ 24 n U_f (2\tau ) ^{D}  \geq   \log(2/\delta )  $, we obtain that  
\begin{align*}
\left| \sum_{i=1} ^n  Z_i  \right| \leq  2 \sqrt {  \frac{ n U_f (2\tau ) ^{D}}{6}  \log(2/\delta) } .
\end{align*}
Replacing $\delta$ by $\delta / D^2$ and using the union bound, we get the desired result.

\end{proof}

An important quantity in the framework we develop is 
\begin{align*}
    \sum_{i : X_i \in \mathcal B (x,  \hat \tau_k (x)) } (X_{i,j}-x_j)   ,
\end{align*}
for which we provide an upper bound in the following theorem. Note that we improve upon the straightforward bound of $ k \hat \tau_k (x)$ which is unfortunately not enough for the analysis carried out here.
We shall work with the following assumption 
\begin{align}\label{eq:cond_for_talagrand}
& C_1  \left( D\log(  2 n  )  +    \log( 2 n / \delta)\right)   \leq k  \leq  C_2  n ,
\end{align}
where the two constants $C_1>0$ and $C_2>0$ are given in the following proposition.

\begin{proposition}\label{prop:key_lemma}
 Suppose that Assumption \ref{cond:density} and \ref{cond:lip3} are fulfilled. Let $n\geq 1$, $k\geq 1$ and $\delta\in (0,1)$. There exist universal positive constants $C_1$, $C_2$, and $C_3$ such that, under \eqref{eq:cond_for_talagrand}, we have with probability $1-\delta$,
\begin{align*}
\max_{j=1,\ldots, D}  \left|   \sum_{i : X_i \in \mathcal B (x,  \hat \tau_k(x) ) } (X_{i,j}-x_j)   \right|\leq     C_3    \left( \overline \tau_k     \sqrt {   k D \log( 2 n D  /  \delta )  }  +  \frac { L k\overline\tau_k   ^{2}}{ b_f  }   \right).
\end{align*}
\end{proposition}

\begin{proof}
Taking $C_1$ greater than $4$, we ensure that $k\geq 4   \log(2n/\delta)  $. Taking $C_2$ small enough, we guarantee that $ \overline \tau_k \leq \tau_ 0 $.
From Proposition \ref{prop:tau}, we have that $\hat \tau_ k(x) \leq \overline \tau _k $ is valid with probability $1-\delta/2$.

Let $ \mu(  \tau  ) =  \mathbb E [  (X_{1} - x )  \mathds{1}_{\mathcal B (x,     \tau ) }(X_1)] $.  Consider the following decomposition
\begin{align*}
  |  \sum_{i : X_i \in \mathcal B (x,  \hat \tau_k(x) ) } (X_{i,j}-x_j)  |& \leq \left|    \sum_{i = 1 }^n  \{ (X_{i,j}-x_j)\mathds{1}_{\mathcal B (x,     \hat \tau_k  (x) ) }(X_{i,j})  -  \mu_ j (  \hat \tau_k (x) )\}\right| +   n \mu _ j ( \hat \tau_ k(x) )   \\
  &\leq  \sup_{ 0 < \tau\leq \overline \tau_k } \left|    \sum_{i = 1 }^n  \{ (X_{i,j}-x_j)\mathds{1}_{\mathcal B (x,    \tau   ) }(X_{i,j})  -  \mu_ j (   \tau )\}\right| +  n \mu _ j ( \hat \tau_ k(x) ) .
\end{align*}
Notice that
\begin{align*}
\mu(  \tau  ) & = \int    ( y -x)   \mathds{1}_{\mathcal B (x,       \tau  ) }(y) f(y)d y=   (2  \tau  ) ^{1+D}  \int_{\mathcal B (0,1/2) } v f(x+\tau v)   dv\\
&=   (2  \tau  ) ^{1+D}  \int_{\mathcal B (0,1/2) } v (f(x+\tau v) - f(x))   dv.
\end{align*}
Hence
\begin{align*}
|\mu _ j (  \tau  ) |&\leq \frac L 2    (2  \tau  ) ^{2+D}  \int_{\mathcal B (0,1/2) } v_j    |v|_\infty   dv \leq \frac L 8   (2  \tau  ) ^{2+D} = \frac L 8     (2  \tau  ) ^{2+D}   .
\end{align*}
And we find
\begin{align*}
\sup_{j=1,\ldots, D} |\mu_j (  \hat \tau_ k)  |  &\leq  \frac L 8     (2  \overline\tau_k  ) ^{2+D}  =   \frac {L k}{ b_f n }        \overline\tau_k   ^{2}.
\end{align*}
The class of rectangles $\{y\mapsto \mathds{1}_{\mathcal B (x ,     \tau    )}(y) \,:\, \tau >0 \}$  has a VC dimension smaller that $v = (2D+1)$ \citep[Proposition 2.3]{wenocurSpecialVapnikchervonenkisClasses1981}. 
From Theorem 2.6.4 in \cite{vaartWeakConvergenceEmpirical1996}, we have
\begin{align*}
\mathcal N  \left( \epsilon , \mathcal R , L_2 (Q)  \right)\leq Kv (4e) ^v \left( \frac 1 \epsilon \right)  ^{2 (v- 1) }  
\end{align*}
for any probability measure $Q$.
This implies that $ \mathcal N  \left( \epsilon , \mathcal R , L_2 (Q)  \right) \leq   \left(  {A} / { \epsilon} \right)  ^{4D }$,
where $A$ is a universal constant. As a result, the class 
\begin{align*}
\mathcal F_j = \left\{y\mapsto \frac{ (y-x_j )  }{  \overline \tau_k}    \mathds{1}_{\mathcal B (x,     \tau )}(y)  \, : \, \tau \in (0, \overline \tau_k]\right\},
\end{align*}
which is uniformly bounded by $1$,  satisfies the same bound for its covering number, that is
\begin{align*}
\mathcal N  \left( \epsilon , \mathcal F_j , L_2 (Q)  \right) \leq   \left( \frac {A}{ \epsilon} \right)  ^{4D }.
\end{align*}
We can therefore apply Lemma \ref{vc_bound} with $v = 4D$, $A$ a universal constant, $U=1$ and $\sigma^2 $ defined as
\begin{align*}
\Var\left(  \frac{ (X_1-x)_j }{  \overline \tau_k}    \mathds{1}_{\mathcal B (x,     \tau )}(X_1) \right)\leq \mathbb E [  \mathds{1}_{\mathcal B (x,     \tau )} (X_1)] \leq  \mathbb E [  \mathds{1}_{\mathcal B (x,    \overline  \tau_k )}(X_1) ]\leq \frac{2U_f}{b_f} \frac{k}{n}\leq \frac{4k}{n}
: = \sigma^2 .
\end{align*} 
Condition \ref{eq:constant_n_delta_1} and \ref{eq:constant_n_delta_2} are valid under \eqref{eq:cond_for_talagrand} when $C_1$ is a large enough constant. The fact that $\sigma^2\leq 1$ is provided by  \eqref{eq:cond_for_talagrand} as well. We obtain that 
\begin{align*}
 \sup_{ 0 < \tau\leq \overline \tau_k } \left|    \sum_{i = 1 }^n  \{ (X_{i,j}-x_j)\mathds{1}_{\mathcal B (x,    \tau   ) } (X_{i,j}) -  \mu_ j (   \tau )\}\right| 
 & \leq \overline \tau_k   C  \sqrt {   k D \log(2 n   /  \delta )  } ,
\end{align*}
where $C$ is a universal constant ($C$ should be large enough to absorb the other constants involved until now). Using the union bound, this bound is extended to a uniform bound over $j\in\{1,\ldots,D\}$. We then obtain the statement of the proposition.
\end{proof}
	
\subsection{Proof of Theorem \ref{theorem:loco}}

We rely on the bias-variance decomposition expressed in \eqref{decomp_bias_var}. On the first hand, we have
\begin{align*}
| m _k(x) -  m  (x)| &= \left|\frac{\sum_{i=1} ^ n ( m  (X_i) -   m (x))  \mathds{1}_{\{  \mathcal{B}(x,\hat \tau_{k}(x)) \} } (X_i)  }{\sum_{i=1} ^ n \mathds{1}_{\{  \mathcal{B}(x, \hat \tau_{k}(x)) \} } (X_i)  } \right| \\
& \leq \sup_{y\in \mathcal{B}(x, \hat \tau_k(x) ) } |  m (y) -  m (x)|\\
&\leq  L_1 \hat \tau_{k} (x).
\end{align*}
Applying Lemma \ref{prop:tau} we obtain that, with probability $1-\delta/2$,
\begin{align*}
| m _k(x) - m (x)|  \leq L_1  \overline \tau_{k} .
\end{align*}
On the other hand, we apply Proposition \ref{prop:var1} to get that, with probability $1-\delta/2$,
\begin{align*}
|\hat  m _k(x) -  m _k (x)|\leq \sqrt{ \frac{2  \sigma^2 \log(4/\delta)}{k} } .
\end{align*}

\subsection{Proof of Theorem \ref{th:gradient_ell2}}

Denote by $\mathbb X$ the design matrix of the (local) regression problem 
\begin{align*}
&\mathbb X = ( X_{i_1}^c , \ldots, X_{i_k}^c )^T \\
&\mathbb Y = (y_{i_1}^c, \ldots , y_{i_k}^c)^T.
\end{align*}
 where for any $j=1,\ldots, k$, $i_j$ is such that $X_{i_j}\in \mathcal B(x;\hat \tau_k(x))$. Define  $ w = \mathbb Y- \mathbb X \beta^*$,  $\hat \nu = \hat \beta_k(x) -\beta^*$
Following \cite{hastieStatisticalLearningSparsity2015}, define
\begin{align*}
\mathcal C(S,\alpha) = \{u\in \mathbb R^D \, : \,  \|u_{\overline S}\|_1 \leq \alpha \|u_{S}\|_1\}. 
\end{align*}
and let $\hat \gamma_n$ be defined as
\begin{align*}
\hat \gamma_n = \inf _ {u\in \mathcal C(S,3)} \frac{ \|   \mathbb X u \|_2^2}{k\|u\|_2^2}  .
\end{align*}
Hence,  $\hat \gamma_n $ is the smallest eigenvalue (restricted to the cone) of the design matrix $\mathbb X$. From Lemma 11.1 in \cite{hastieStatisticalLearningSparsity2015}, we have the following:
whenever
\begin{align*}
\lambda \geq (2 / k) \| \mathbb X^T w\|_\infty,
\end{align*}
it holds that
\begin{align*}
& \hat \nu  \in \mathcal C(S,3),\\
 &\| \hat \nu\|_2  \leq \frac{3\sqrt {|\mathcal S_x|}  }{   \hat \gamma_n}  \lambda.
\end{align*}
Consequently, the proof will be completed if,  with probability $1-\delta$,
\begin{align}
\label{eq:bound_Xw} &(2/k) \|   \mathbb X^T_j w \|_\infty \leq    \overline \tau_k   \left( \sqrt{ \frac{2  \sigma^2   \log(16D/\delta) }{k} } +  L_2 \overline \tau_k  ^2\right),\\
\label{eq:bound_gamma} &\hat \gamma_n  \geq   \frac { \overline \tau_k  ^{2 }    }{24\times 8  } .
\end{align}

\paragraph{Proof of \eqref{eq:bound_Xw}.}  In the next few lines, we show that \eqref{eq:bound_Xw} holds with probability $1-\delta/2$. By definition 
\begin{align*}
\mathbb X^T w =  \sum_{i : X_i \in \mathcal B (x, \hat \tau_k(x)) }  w_i^c  X_{i}^c =\sum_{i : X_i \in \mathcal B (x, \hat \tau_k(x)) } w_i X_{i} ^c  .
\end{align*}
Using that $w_i = \xi_i+ m(X_i) - \beta^{*T} X_i$,
\begin{align*}
   \mathbb X^T w  & = \sum_{i : X_i \in \mathcal B (x, \hat \tau_k(x)) } X_{i} ^c \xi_i   +  \sum_{i : X_i \in \mathcal B (x, \hat \tau_k(x)) } X_{i} ^c (m(X_i) - \beta^{*T}X_i)\\
 &= \sum_{i : X_i \in \mathcal B (x, \hat \tau_k(x)) } X_{i} ^c \xi_i    +\sum_{i : X_i \in \mathcal B (x, \hat \tau_k(x)) } X_{i} ^c (m(X_i) - m(x) - \beta^{*T}(X_i-x) )
\end{align*}
where we have used the covariance structure (with empirically centred terms) to derive the last line.
Note that for any $\tau >0$, $\max_{i : X_i \in \mathcal B (x,  \tau )} |  X_{i,j} ^c| \leq  \tau$.  Hence, from Proposition \ref{prop:var2}, because $\overline \tau_k\leq \tau_0$ and $k\geq 4   \log( 8D n/\delta)$ (taking $C_1$ large enough), we have with probability $1-\delta / (4D)$,
\begin{align*}
\left|  \sum_{i : X_i \in \mathcal B (x, \hat \tau_k(x)) } X_{i,j} ^c \xi_i   \right| \leq \sqrt{ 2 k \sigma^2 \overline \tau_k ^2 \log(16D/\delta) }
\end{align*}
Moreover,
\begin{multline*}
\sum_{i : X_i \in \mathcal B (x, \hat \tau_k(x)) } | X_{i,j} ^c|  | m (X_i) -  m (x) - g(x) ^{T}(X_i-x) | \leq\\
 k L_2 \hat \tau_k (x) ^2 \max_{i : X_i \in \mathcal B (x, \hat \tau_k(x) )} |  X_{i,j} ^c| \leq k L_2 \hat \tau_k (x) ^3
\end{multline*}
Using Proposition \ref{prop:tau}, because $k\geq 4   \log( 4D n/\delta) $, it holds, with probability $1-\delta/(4D)$,
\begin{align*}
\sum_{i : X_i \in \mathcal B (x, \hat \tau_k(x)) } | X_{i,j} ^c|  | m (X_i) -  m (x) - \beta^{*T}(X_i-x) | \leq  k L_2 \overline \tau_k  ^3
\end{align*}
We finally obtain that for any $j=1,\ldots, D$, it holds, with probability $1-\delta/(2D)$,
\begin{align*}
|   \mathbb X^T_j w | \leq  \sqrt{2 k \sigma^2 \overline \tau_k ^2 \log(16/\delta) } +   k L_2 \overline \tau_k ^3,
\end{align*}
and from the union bound, we deduce that, with probability $1-\delta/2$,
\begin{align*}
\max_{j=1,\ldots, D} |   \mathbb X^T_j w | \leq  \overline \tau_k   \left( \sqrt{2 k  \sigma^2   \log(16D/\delta) } + kL_2 \overline \tau_k  ^2\right).
\end{align*}

\paragraph{Proof of \eqref{eq:bound_gamma}.} We show that \eqref{eq:bound_gamma} holds with probability $1-\delta/ 2 $. Define 
\begin{align*}
&\hat \Sigma_k =  \sum_{i : X_i \in \mathcal B (x,    \underline \tau_k ) } (X_i-x  )(X_i-x )^T.
&\hat \mu(\tau)  =  \sum_{i : X_i \in \mathcal B (x,    \tau ) } (X_i-x  ).
\end{align*} 
First, note that
\begin{align*}
  \mathbb X^T \mathbb X   &=   \sum_{i : X_i \in \mathcal B (x, \hat \tau_k(x)) }      (X_i -x) (X_i -x) ^T - k^{-1}  \hat \mu( \hat \tau_ k )\hat \mu( \hat \tau_ k )  ^{ T}  .
\end{align*}
Then, using Proposition \ref{prop:tau2}, because $ k\geq 4   \log( 4n/\delta)  $, with probability $1-\delta/4$, $\hat \tau_k(x)\geq \underline \tau_k$, implying that
\begin{align*}
    \mathbb X^T \mathbb X &\geq  \hat \Sigma_k   - k^{-1}  \hat \mu( \hat \tau_ k )\hat \mu( \hat \tau_ k )  ^{ T}   =     \mathbb E [\hat \Sigma_k  ] +  (\hat \Sigma_k  - \mathbb E [\hat \Sigma_k  ] )  - k^{-1}  \hat \mu( \hat \tau_ k )\hat \mu( \hat \tau_ k )  ^{ T} 
\end{align*} 
Let $u\in \mathbb R^D$. We have that
\begin{align*}
|u^T \hat \mu( \hat \tau_ k ) |^2  &\leq  \|u\|_1^2 \max_{j=1,\ldots, D}  | (\hat \mu( \hat \tau_ k ))_j |^2\leq  |\mathcal S_x| \|u\|_2^2  \max_{j=1,\ldots, D}  | (\hat \mu( \hat \tau_ k ))_j |^2.
\end{align*}
Similarly, we have:
\begin{align*}
 | u^T (\hat \Sigma_k - \mathbb E \hat \Sigma_k)  u |  &\leq  \|u\|_1^2 \| \hat \Sigma_k - \mathbb E \hat \Sigma_k\|_\infty    \leq   |\mathcal S_x| \|u\|_2^2 \| \hat \Sigma_k - \mathbb E \hat \Sigma_k\|_\infty   .
\end{align*}
Using the variable change $y = x+ 2\underline \tau_k v$ and that $ \underline \tau_k    \leq  \tau_0  $, we have that
\begin{align*}
\mathbb E \hat \Sigma_k  &= n \mathbb E [  (X_1-x) (X_1-x)^T 1_{ \mathcal B (x, \underline    \tau_k) }(X_1)]  = n\int    ( y -x) ( y -x) ^T 1_{\{y\in \mathcal B (x,    \underline  \tau_k ) \}} f(y)d y\\
&\geq n   b_f \int    ( y -x) ( y -x) ^T 1_{\{ y\in\mathcal B (x,    \underline  \tau_k ) \}} d y=  n (2\underline \tau_k ) ^{2+D}  \int_{v\in \mathcal B (0,1/2) } vv^T   dv\\
& =    n(2\underline \tau_k ) ^{2+D} b_f\left(  \int_{  [-1/2,1/2] } v_1^2  dv_1 \right) I_D  \\
&=  \frac {b_f }{12} n  (2\underline \tau_k ) ^{2+D}    I_D = \frac {b_f }{24 U_f}   \underline \tau_k  ^{2 } k    I_D \geq \frac { \underline \tau_k  ^{2 } k  }{48 }     I_D,
\end{align*}
using that $U_f/b_f \leq 2$. Consequently,
\begin{align*}
   u^T \mathbb X^T \mathbb X u &\geq  \frac {b_f }{12} n  (2\underline \tau_k ) ^{2+D}    -     |\mathcal S_x| \left( \| \hat \Sigma_k - \mathbb E \hat \Sigma_k\|_\infty    + k^{-1} \max_{j=1,\ldots, D}  | (\hat \mu( \hat \tau_ k ))_j |^2\right).
  \end{align*} 
 Note that $ \overline \tau_k = C_f ^{1/D} \underline \tau_k$ with $C_f\leq 8$.
Proposition \ref{prop:cov_bound} can be applied because $24 nU_f  (2\underline \tau_k ) ^{D} = 12k  \geq   \log(8D^2 /\delta ) $ which is satisfied whenever $C_1$ is large. Combined with Proposition \ref{prop:key_lemma}, we obtain that, with probability $1-\delta/4$,
\begin{align*}
\frac{\|\mathbb X u \|_2^2 }{\|u\|_2^2}  & \geq \frac { \underline \tau_k  ^{2 } k  }{12 }   -  |\mathcal S_x| \left( 4 \underline \tau_k^2  \sqrt {  \frac{  k }{3}   \log(16 D^2/\delta)  }+
2C_3 ^2   \left( \overline \tau_k^2      D \log(  16 n D  /  \delta )    +  \frac { L^2 k \overline\tau_k   ^{4}}{ b_f^2  }   \right).    
 \right)\\
&\geq   \frac { \overline \tau_k  ^{2 } k   }{24\times 8  } \left( 2    - |\mathcal S_x|C \left(    \sqrt {  \frac{   \log(2D/\delta)   }{k}  }+
\frac{  D \log( 2 n  D  /  \delta )} {k}   +  \frac {      \overline\tau_k   ^{2} L ^2}{ b_f^2  }   
 \right) \right )    .
\end{align*}
Choose $C_1 $ large enough to get that $ C |\mathcal S_x| \sqrt{  \log( 2  D  /  \delta )}  \leq \sqrt k / 3 $ and $ C|\mathcal S_x|  D \log( 2 n  D  /  \delta )  \leq  k / 3 $. Finally, noting that
 $  C |\mathcal S_x|     \overline\tau_k   ^{2} L ^2    \leq   b_f^2  /3$ we obtain the result.

\section*{Technical Appendix}\label{sec:technical}

We present here for reference some technical elements of the applications that while not needed to understand the results can be of interest to some readers. All the relevant implementation details can be found in the \texttt{Julia} code in the supplementary materials.

We use a standard symmetrical encoder-decoder architecture for the variational autoencoder, schematically presented in Figure~\ref{fig:vae}. 
\begin{figure}[hbt!]
    \centering
    \includegraphics{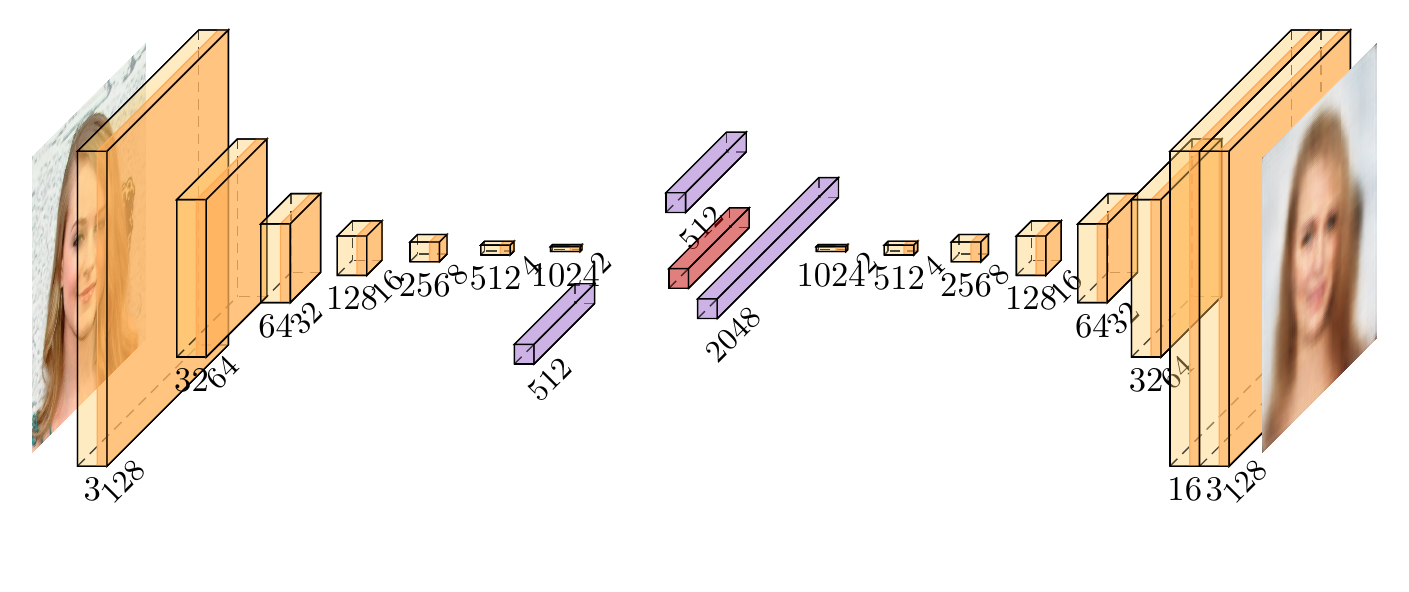}
    \caption{Encoder-Decoder Architecture used for this work}\label{fig:vae}
\end{figure}
Once all the images have been encoded in $\mathbb{R}^{512}$ it is possible to use the local linear estimator of the gradient studied in this work to derive the gradient of the age with respect to the latent variable, making it possible to produce a new version of the input image that appears either older or younger as done in Figure~\ref{fig:age}. By computing a local estimate of the gradient, we are able to derive a more meaningful change when the age is not perfectly disentangled.
\begin{figure}[hbt!]
    \centering
    \begin{tikzpicture}
        \node[inner sep=0pt] (encoded) at (0,0)
            {\includegraphics[width=.15\textwidth]{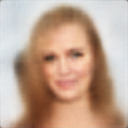}};
        \node[inner sep=0pt] (aged) at (6,0)
            {\includegraphics[width=.15\textwidth]{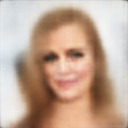}};
        \draw[->,thick] (encoded.east) -- (aged.west)
            node[midway,fill=white] {$z + 0.1 \times \nabla m(z)$};
    \end{tikzpicture}
    \caption{Extracting the direction of interest for aging.}\label{fig:age}
\end{figure}
Note that the quality of the image reconstruction and generation is here solely limited by the choice of the encoding and decoding model and is not related to the methods introduced in this paper, significant advances in the quality of the decoding have been made in the recent years and if a better quality and less blurry decoded output are desired we encourage the reader to replace the decoder with a \texttt{PixelCNN} architecture such as presented in~\cite{salimansPixelCNNImprovingPixelCNN2017}. The quality of the gradient is also significantly impacted by the quality of the annotations as \texttt{CACD200} is an automatically annotated and noisy dataset. 
\end{document}